\documentclass[wcp]{jmlr}

\usepackage{longtable}% for long tables

\usepackage{booktabs}
\usepackage{nicefrac}
\usepackage{tikz}
\usepackage{xcolor}
\usepackage{caption}
\usepackage{floatrow}

\makeatletter
\let\Ginclude@graphics\@org@Ginclude@graphics 
\makeatother

\hypersetup{
	colorlinks=true,
	urlcolor=blue,
	citecolor=blue,
	linktoc=all,
	linkcolor=blue}
\usepackage{graphicx}
\usepackage{mathtools}

\newtheorem{thm}{Theorem}[section]
\newtheorem{prop}{Proposition}[section]
\newtheorem{lm}{Lemma}[section]

\newtheorem{manualtheoreminner}{Theorem}
\newenvironment{manualtheorem}[1]{%
  \renewcommand\themanualtheoreminner{#1}%
  \manualtheoreminner
}{\endmanualtheoreminner}

\newtheorem{manuallemmainner}{Lemma}
\newenvironment{manuallemma}[1]{%
  \renewcommand\themanuallemmainner{#1}%
  \manuallemmainner
}{\endmanuallemmainner}

\newtheorem{rem}{Remark}[section]

\def\GWTR{\mathrm{GWT}_\mathbb{R}}

\def\theoremSumNGwt{Let $X_1, \dots, X_N\in\GWTR$ with tail parameters $\beta_{1}, \dots, \beta_{N}$. If $X_1, \dots, X_N$ satisfy the PD condition of Definition~\ref{def:positive_dependence_condition}, then, $X_1 + \dots + X_N\in\GWTR(\beta)$ with $\beta = \min\{\beta_{1}, \dots, \beta_{N}\}$.}

\def\theoremProductGwt{Let $X_1, \dots, X_N\in\GWTR$ be independent symmetric with tail parameters $\beta_{1}, \dots, \beta_{N}$. 
Then, the product $X_1\ldots X_N\in\GWTR(\beta)$ with $\beta$ such that $\frac 1 \beta = \frac 1 {\beta_1} + \cdots+ \frac 1 {\beta_N}$.}

\def\lemmaHiddenUnitsDependence{Let $X_1, \dots, X_N$ be some possibly dependent random variables and $W_1, \dots, W_N$ be symmetric, mutually independent and independent from $X_1, \dots, X_N$, then random variables $X_1 W_1, \dots, X_N W_N$ satisfy the PD condition.
}

\def\theoremHiddenUnitsAreGWT{Consider a Bayesian neural network as described in~Equation~\eqref{eq:hidden_unit_form} with ReLU activation function. Let $\ell$-th layer weights be independent symmetric generalized Weibull-tail on~$\mathbb{R}$ with tail parameter $\beta^{(\ell)}_w$. 
Then, $\ell$-th layer pre-activations are generalized Weibull-tail on~$\mathbb{R}$ with tail parameter $\beta^{(\ell)}$ such that $\frac{1}{\beta^{(\ell)}} = \frac{1}{\beta^{(1)}_w} + \dots + \frac{1}{\beta^{(\ell)}_w}$. }

\newcommand{\edr}{\mathrm{e}} 
\newcommand{\ddr}{\mathrm{d}}

\jmlrvolume{}
\jmlryear{2021}
\jmlrworkshop{ACML 2021}

\title{
Bayesian neural network unit priors \ \\
 and generalized Weibull-tail property}
  \author{\Name{Mariia Vladimirova} \Email{mariia.vladimirova@inria.fr}\\
  \Name{Julyan Arbel} \Email{julyan.arbel@inria.fr}\\
  \Name{St\'ephane Girard} \Email{stephane.girard@inria.fr}\\
  \addr Univ. Grenoble Alpes, Inria, CNRS, LJK, 38000 Grenoble, France
 }

\begin{document}

\maketitle

\begin{abstract}
The connection between Bayesian neural networks and Gaussian processes gained a~lot of attention in the last few years. Hidden units are proven to follow a Gaussian process limit when the layer width tends to infinity. Recent work has suggested that finite Bayesian neural networks may outperform their infinite counterparts because they adapt their internal representations flexibly. To establish solid ground for future research on finite-width neural networks, our goal is to study the prior induced on hidden units. Our main result is an accurate description of hidden units tails which shows that unit priors become heavier-tailed going deeper, thanks to the introduced notion of generalized Weibull-tail. This finding sheds light on the behavior of hidden units of finite Bayesian neural networks. 
\end{abstract}

\section{Introduction}

Theoretical insights and the development of a comprehensive theory are often the driving force underlying the development of new and improved methods. Neural networks are powerful models, but they still lack comprehensive theoretical support. The Bayesian approach applied to neural networks is considered one of the best-suited frameworks to obtain theoretical explanations and improve the models. 

Infinite-width Bayesian neural networks are well-studied. Induced priors in Bayesian neural networks with Gaussian weights are Gaussian processes when the number of hidden units per layer tends to infinity~\citep{neal1996bayesian,matthews2018gaussian,lee2018deep,garriga2019deep}. Stable distributions also lead to stable processes, which are generalizations of Gaussian ones~\citep{favaro2020stable}.

Since in some cases finite models perform better~\citep{lee2018deep, garriga2019deep, lee2020finite}, there is a need for theoretical justifications. One of the ways is to study induced priors. By analyzing the priors over representations, \citet{aitchison2020bigger} suggests that finite Bayesian neural networks may generalize better than their infinite counterparts because of their ability to learn representations. Another idea is to find the induced priors in the functional space~\citep{wilson2020bayesian}. 

By bounding moments of distributions induced on units, \citet{vladimirova2018bayesian} proved that hidden units have heavier-tailed upper bounds and follow sub-Weibull distributions with an increasing tail parameter depending on the layer depth. Those bounds are optimal as they are achieved for shallow Bayesian neural networks; however, they are not accurate. 

Recently, \citet{zavatone2021exact,noci2021precise} showed that there exists a precise description of induced unit priors through Meijer G-functions. These are full descriptions of priors at the unit level. The results are in accordance with the heavy-tailed nature and asymptotic expansions in a wide regime, but it has restrictions. First, the setting is simplified: linear or ReLU activation functions and Gaussian priors on weights. While \citet{wilson2020bayesian} argue that vague Gaussian priors in the parameter space induce useful function-space priors, in some cases, heavier-tailed priors can perform better~\citep{fortuin2021bayesian}. Second, it is hard to work with Meijer G-functions due to their complexity. Our goal is to obtain more general characterizations for hidden units.

We introduce a new concept for describing distributional properties of tails by extending the existing Weibull-tail characterization. A random variable $X$ is called \textit{Weibull-tail}~\citep{gardes2011weibull} with tail parameter $\beta > 0$, which is denoted by $X \sim \text{WT}_{\mathbb{R}_+}(\beta)$, if its cumulative distribution function $F_X$ satisfies
\begin{equation}
\label{eq:weibull-tail_def}
    \overline{F}_X(x) = 1 - F_X(x) = \edr^{-x^{\beta} l(x)}, \,\,\, \text{for }x>0,
\end{equation}
where $l(x)$ is a slowly-varying function, i.e. it is a positive function such that for all $t > 0$
$\lim_{x\to\infty} \frac{l(tx)}{l(x)} = 1$.
We note that the Weibull-tail property only considers the right tail of distributions. Here we adapt the Weibull-tail characterization to the whole space $\mathbb{R}$ by taking into consideration the left tail as well. Additionally, we introduce \textit{generalized Weibull-tail} random variables with tail parameter $\beta > 0$ which have Weibull-tail upper and lower bounds for both tails (Definition~\ref{def:gen_weibull-tail_r}). Such a characterization is  easily interpretable and stable under basic operations, even for dependent random variables. 
The family of generalized Weibull-tail distributions covers a large variety of fundamental distributions such as Gaussian ($\beta = 2$), gamma ($\beta = 1$), Weibull ($\beta > 0$), to name a few, and turns out to be a key tool to describe distributional tails.\vspace{0.5em}

\noindent\textbf{Contributions.} We make the following contributions:
\vspace{-0.15cm}
\begin{itemize}
\setlength\itemsep{-0.1cm}
    \item We introduce a new notion of tail characteristics called \textit{generalized Weibull-tail}, a version of Weibull-tail characteristics on $\mathbb{R}_+$ extended to variables on $\mathbb{R}$~(Section~\ref{section:weibull_tail}). The additional advantage of this notion is stability under basic operations such as multiplication by a constant and summation. 
    
    \item With the results on generalized Weibull-tail characterization and dependence, we obtain an accurate characterization of the heavy-tailed nature of hidden units in Bayesian neural networks (Section~\ref{section:bnn}). We establish these results under possibly heavy-tailed priors and relatively mild assumptions on the non-linearity. The conclusions of~\citet{vladimirova2018bayesian,zavatone2021exact,noci2021precise} which consider only Gaussian priors, mostly follow as corollaries of the obtained characterization. 
    The comparison of different characterizations and related works are deferred to Sections~\ref{section:comparison} and \ref{section:future_applications}. 
\end{itemize}

\section{Generalized Weibull-tail random variables}
\label{section:weibull_tail}

The study of the distributional tail behavior arises in many applied probability models of different areas, such as hydrology~\citep{strupczewski2011tails}, finance~\citep{rachev2003handbook} and  insurance risk theory~\citep{mcneil2015quantitative}. Since exact distributions are not available in most cases, deriving asymptotic relationships for their tail probabilities becomes essential. In this context, an important role is played by so-called Weibull-tail distributions satisfying Equation~\eqref{eq:weibull-tail_def}~\citep{gardes2011weibull,gardes2016estimation}.

A majority of works focuses on right tails of distributions and develops a theory only applicable to right tails, while it is essential to study both right and left tails of distributions. 
We extend the notion of Weibull-tail on $\mathbb{R}_+$ to $\mathbb{R}$ and introduce generalized Weibull-tail random variables on~$\mathbb{R}$ in the following definition. They are stable (under basic operations) extensions of Weibull-tail random variables on~$\mathbb{R}$ (Appendix~\ref{appendix:weibull-tail_properties_proofs}).

\begin{definition}[Generalized Weibull-tail on $\mathbb{R}$]
\label{def:gen_weibull-tail_r}
A random variable $X$ is generalized Weibull-tail on $\mathbb{R}$ with tail parameter $\beta > 0$ if both its right and left tails are upper and lower bounded by some Weibull-tail functions with tail parameter~$\beta$:
\begin{align*}
    \edr^{-x^{\beta} l_1^r(x)} \le & \overline{F}_X(x) \le \edr^{-x^{\beta} l_2^r(x)}, \quad \ \ \ \text{for }x > 0 \text{ and $x$ large enough}, \\
    \edr^{-|x|^{\beta} l_1^l(|x|)} \le & F_X(x) \le \edr^{-|x|^{\beta} l_2^l(|x|)}, \quad \text{for } x < 0 \text{ and $-x$ large enough},
\end{align*}
where $l_1^r$, $l_2^r$, $l_1^l$ and $l_2^l$ are slowly-varying functions. 
We note $X \sim \GWTR(\beta)$. 
\end{definition}

This family includes widely used distributions such as
Gaussian ($\beta = 2$), Laplace ($\beta = 1$) and generalized Gaussian distributions. These distributions are also symmetric (around 0), so their left and right tails are equal. Thus, it naturally leads to obtain tail characteristics for symmetric distributions by considering random variables whose absolute value is Weibull-tail on $\mathbb{R}_+$. See Appendix~\ref{appendix:weibull-tail_properties_proofs} for details. 
While Weibull-tail random variables are also generalized Weibull-tail, the opposite is not always true. Consider slowly-varying functions $l_1\equiv1$ and $l_2\equiv2$. Then function $l(\cdot)=1+\cos^2(\ln(\cdot))$ satisfies $l_1 \le l \le l_2$ but is not slowly varying. However, for any $\beta \ge 2$, function $\overline{F}_X(x) = \edr^{-x^{ \beta} l(x)}$ is the survival function (it is decreasing) of some random variable $X$ and it satisfies $\edr^{-x^{\beta} l_2(x)} \le \overline{F}_X(x) \le \edr^{-x^{\beta} l_1(x)}$. Therefore, $X$ is $\text{GWT}(\beta)$ but not $\text{WT}(\beta)$.

Next, we aim to obtain a tail characterization for the sum of generalized Weibull-tail variables on $\mathbb{R}$, where we allow random variables to be dependent. Further, we show that under the following assumption of positive dependence, the sum has a tail parameter equal to the minimum among the considered ones (Theorem~\ref{theorem:sum_of_N_gwt_rvs_on_r}).

\begin{definition}[Positive dependence condition]
\label{def:positive_dependence_condition}
Random variables $X_1, \dots, X_N$ satisfy the \textbf{positive dependence (PD) condition} if the following inequalities hold for all $z \in \mathbb{R}$ and some constant $C> 0$:
\begin{align*}
\mathbb{P}\left(X_1\ge 0,\ldots,X_{N-1} \ge 0 | X_N \ge z\right) \ge C, \quad %\text{(right tail),}\\ 
\mathbb{P}\left(X_1\le 0,\ldots,X_{N-1} \le 0 | X_N \le z \right) \ge C. % \quad \text{ (left tail).}
\end{align*}
\end{definition}

\begin{rem}
   The choice of zeros and $z$ in the PD condition is arbitrary: one can choose any $z_1, \dots, z_{N}$ instead such that $z_1+\cdots+z_N = z$. The choice of the $N$-th variable within $X_1, \dots, X_N$ is also arbitrary. Besides, if random variables $X_1, \dots, X_N$ are independent with non-zero right and left tails, then they satisfy the PD condition and the constant $C$ is equal to the minimum between $\mathbb{P}(X_1\le 0)\ldots \mathbb{P}(X_{N-1}\le 0)$ and $\mathbb{P}(X_1\ge 0)\ldots \mathbb{P}(X_{N-1}\ge 0)$.
\end{rem}

There is a great variety of dependent distributions that obey this dependence property including pre-activations in Bayesian neural networks (see Lemma~\ref{lemma:units_dependence_condition} for details). Note that the positive orthant dependence condition \citep[POD, see][]{nelsen2007introduction} implies the~PD~condition.

\begin{thm}[Sum of GWT$_\mathbb{R}$ variables]
\label{theorem:sum_of_N_gwt_rvs_on_r} 
\theoremSumNGwt
\end{thm}

All proofs are deferred in Appendix. The intuition of the PD condition is to prevent the tail of a sum from becoming lighter due to a negative connection. The simplest example of such negative dependence comes with counter-monotonicity where $(X,Y)$ is such that $Y=-X$. Another less trivial example is $Y = X \mathbb{I}(|X| \le m) - X \mathbb{I}(|X| > m)$ for some $m > 0$: the sum $X + Y = 2 X \mathbb{I}(|X| \le m)$ is a version of $X$ truncated to the compact set $[-m, m]$. In both cases, it is easy to see that the PD condition does not hold.  
We conclude this section with a result on the product of independent generalized Weibull-tail random variables. 

\begin{thm}[Product of $\text{GWT}_{\mathbb{R}}$ variables]
\label{theorem:product_of_double_weilbull-tail_rvs} 
\theoremProductGwt
\end{thm}

Now the obtained results can be applied to Bayesian neural networks, showing that hidden units are generalized Weibull-tail on $\mathbb{R}$. 

%============================================================
\section{Bayesian neural networks induced priors}
\label{section:bnn}
%============================================================
Neural networks are hierarchical models made of layers: an input, several hidden layers, and an output. Each layer following the input layer consists of units which are linear combinations of previous layer units transformed by a nonlinear function, often referred to as the nonlinearity or activation function denoted by~$\phi$. Given an input $\mathbf{x} \in \mathbb{R}^{H_0}$, the $\ell$-th hidden layer consists of a vector whose size is called the width of the layer, denoted by $H_\ell$. The coefficients of linear combinations are called weights, denoted by $w_{ij}^{(\ell)}$ for $i=1, \dots, H_{\ell - 1}, j = 1, \dots, H_\ell$. In Bayesian neural networks, weights are assigned some prior distribution~\citep{neal1996bayesian}. 
The pre-activations and post-activations of layer $\ell$ are respectively defined as
\begin{equation}
\label{eq:hidden_unit_form}
    g_j^{(\ell)} = \sum_{i = 1}^{H_{\ell-1}}  w_{ij}^{(\ell)} h_i^{(\ell - 1)}, \quad  h_j^{(\ell)} = \phi(g_j^{(\ell)})
\end{equation}
where $h_i^{(0)}$ are elements of input vector $\mathbf{x}$, so $h_i^{(0)}$ are deterministic numerical object features.

The main ingredient for Theorem~\ref{theorem:sum_of_N_gwt_rvs_on_r} is the positive dependence condition of Definition~\ref{def:positive_dependence_condition}.  The product of hidden units and weights satisfies the positive dependence condition: 
\begin{lm}
\label{lemma:units_dependence_condition}
\lemmaHiddenUnitsDependence
% When $W_1, \dots, W_N$ are independent, there is equality.
\end{lm}

Along with Theorem~\ref{theorem:sum_of_N_gwt_rvs_on_r}, the previous lemma implies that neural network hidden units are generalized Weibull-tail with tail parameter depending on those of the weights.

\begin{thm}[Hidden units are GWT]
\label{theorem:hidden_units_are_gwt}
\theoremHiddenUnitsAreGWT
\end{thm}
Note that the most popular case of weight prior, iid Gaussian \citep{neal1996bayesian}, corresponds to $\text{GWT}_{\mathbb{R}}(2)$ weights. This leads to units of layer $\ell$ which are $\text{GWT}_{\mathbb{R}}(\frac{2}{\ell})$. 

To illustrate this theorem, we have built neural networks of 4 hidden layers, with 4 hidden units on each layer. 
We used a fixed input $\mathbf{x}$ of size $10^4$, which can be thought of as an image of dimension $100\times 100$. This input was sampled once for all with standard Gaussian entries. 
In order to obtain samples from the prior distribution of the neural network units, we have sampled the weights from independent centered Gaussians from which units were obtained by forward evaluation with the ReLU non-linearity. This process was iterated $n =10^6$ times. 
Note that for a $\GWTR$ random variable, $\mathbb{P}(X \ge x) = \edr^{-x^\beta l(x)}$ so the tail parameter can be expressed as:
\begin{equation}
\label{eq:beta_parameter}
    \beta = \frac{\log (- \log \mathbb{P}(X \ge x))}{\log x} - \frac{\log l(x)}{\log x}.
\end{equation}
In Figure~\ref{pic:nn_layers}, we plot $\log (- \log \mathbb{P}(X \ge x))$ as a function of $\log x$. We see that the obtained tail parameters approximations are increasing for the increasing layer number and visually correspond to the theoretical tail parameter. 

\begin{figure}[h!]
\floatbox[{\capbeside\thisfloatsetup{capbesideposition={left,top},capbesidewidth=.45\textwidth}}]{figure}[\FBwidth]
{\caption{\textit{Solid lines}: approximations of tail parameters $\beta^{(\ell)}$ based on Equation~\eqref{eq:beta_parameter} where $X$ are hidden units of layers $\ell = 1,2,3,4$ corresponding theoretically to generalized Weibull-tail with tail parameters $\beta^{(\ell)} = 2, 1, \nicefrac23, \nicefrac12$,
under the independent Gaussian weights assumption. 
  \textit{Dashed lines}: linear regressions with coefficients equal to the theoretical tail parameters $\beta^{(\ell)} = \nicefrac{2}{\ell}$ and manually selected biases to approach the solid lines for visual comparison.}  \label{pic:nn_layers}
}
{\vspace{-0.55cm}\includegraphics[width=.55\textwidth]{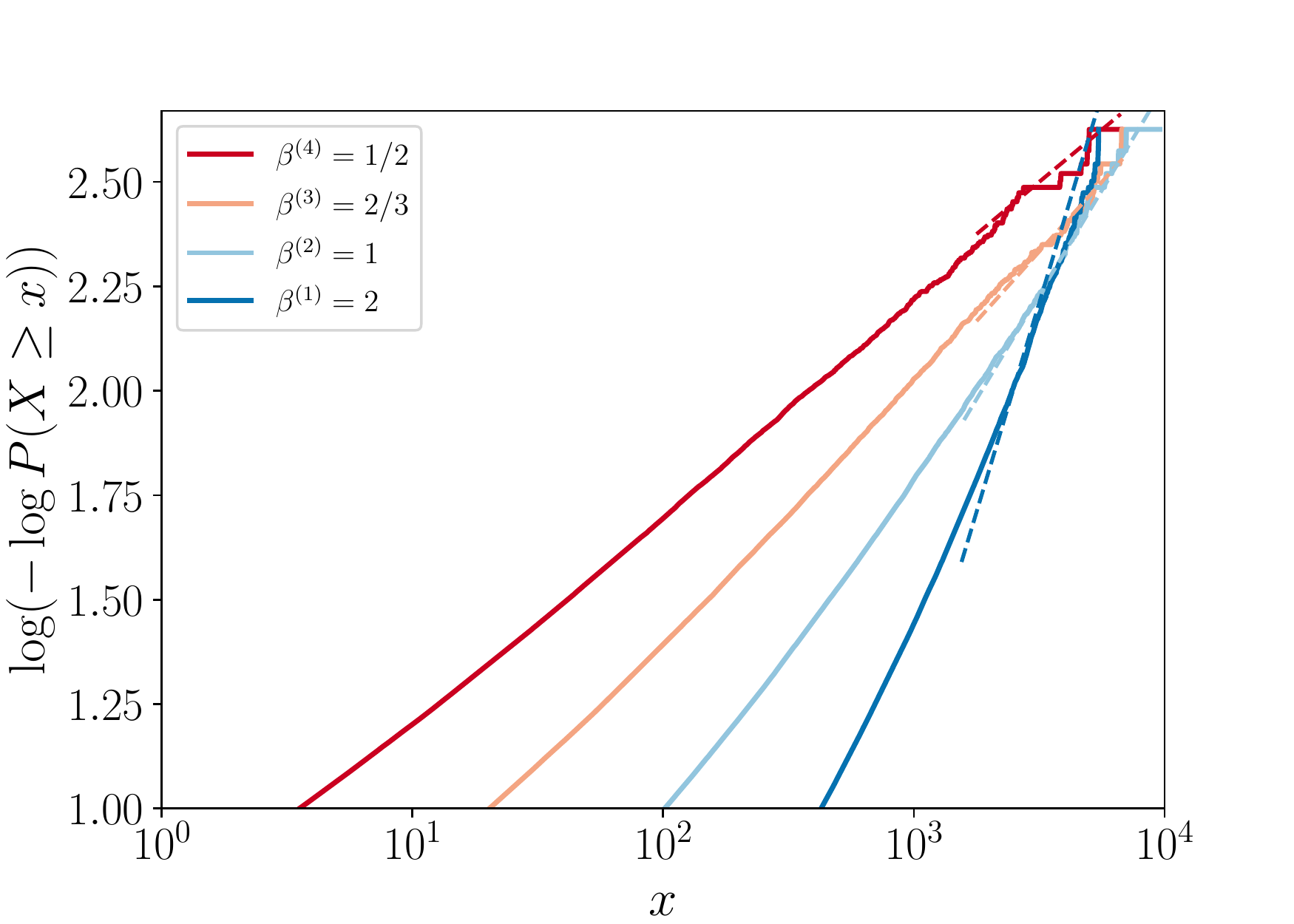}}
\end{figure}

%============================================================
%============================================================

\section{Comparison of different characterizations}
\label{section:comparison}

\subsection{Generalized Weibull-tail vs sub-Weibull}

Some of the commonly used techniques to study the tail behavior is to consider  probability tail bounds such as sub-Gaussian, sub-exponential, or their generalization to sub-Weibull distributions~\citep{vladimirova2020sub,kuchibhotla2018moving}. A non-negative random variable $X$ is called sub-Weibull with tail parameter $\theta > 0$ if its survival function is upper-bounded by that of a Weibull distribution:
\begin{align}\label{eq:sub-W-tail-def}
    \overline{F}_X(x) \le a \edr^{-bx^{1/\theta}},\,\,\, \text{for }x>0\,\,\, \text{and some } a,b>0.
\end{align}
This property ensures the existence of the moment generating function as well as bounds on moments. In contrast, the Weibull-tail property characterizes the survival or density functions without a hand on moments. 
While tail parameters in Equation~\eqref{eq:weibull-tail_def} and~\eqref{eq:sub-W-tail-def} of generalized Weibull-tail and sub-Weibull properties respectively are different, there exist connections. Notice that for any constants $a, b, \beta > 0$, function $l(x) = b - \frac{\log a}{x^\beta} \ge 0$ is slowly-varying for $x$ large enough and $a \edr^{-b x^\beta} = \edr^{-x^\beta l(x)}$. It means that if a random variable $X$ is sub-Weibull with parameter $\theta = 1/\beta > 0$, satisfying Equation~\eqref{eq:sub-W-tail-def}, then the survival function of $X$ is upper-bounded by a Weibull-tail function with tail parameter~$\beta$ and slowly-varying function $l(x) = 1$, satisfying Equation~\eqref{eq:weibull-tail_def}. 
If random variable $X$ is generalized Weibull-tail with tail parameter~$\beta$, then 
from the last item of Proposition~\ref{appendix:proposition:sv_properties}, for $a_1, a_2 > 0$ we have 
\begin{equation*}
     a_1 \edr^{-x^{\beta_1}} \le \overline{F}_X(x) = \edr^{-x^{ \beta} l(x)} \le a_2 \edr^{-x^{\beta_2}},
\end{equation*}
or {$\text{GWT}_{\mathbb{R}_+}(\beta) \subset \text{SubW}(1/\beta_2)$ and  $\text{GWT}_{\mathbb{R}_+}(\beta) \not\subset \text{SubW}(1/\beta_1)$}
for $x$ large enough and $\forall (\beta_1,\beta_2)$ such that $0 < \beta_2 < \beta < \beta_1$, as illustrated on Figure~\ref{fig:potatos}.

% \vspace{-1cm}
\begin{figure}[ht]
\vspace{-1cm}
{\small
\begin{minipage}{0.49\textwidth}
\centering
    \begin{tikzpicture}
    \node[inner sep=0pt] (potatoes) at (0,0)
    {\includegraphics[width=6.6cm]{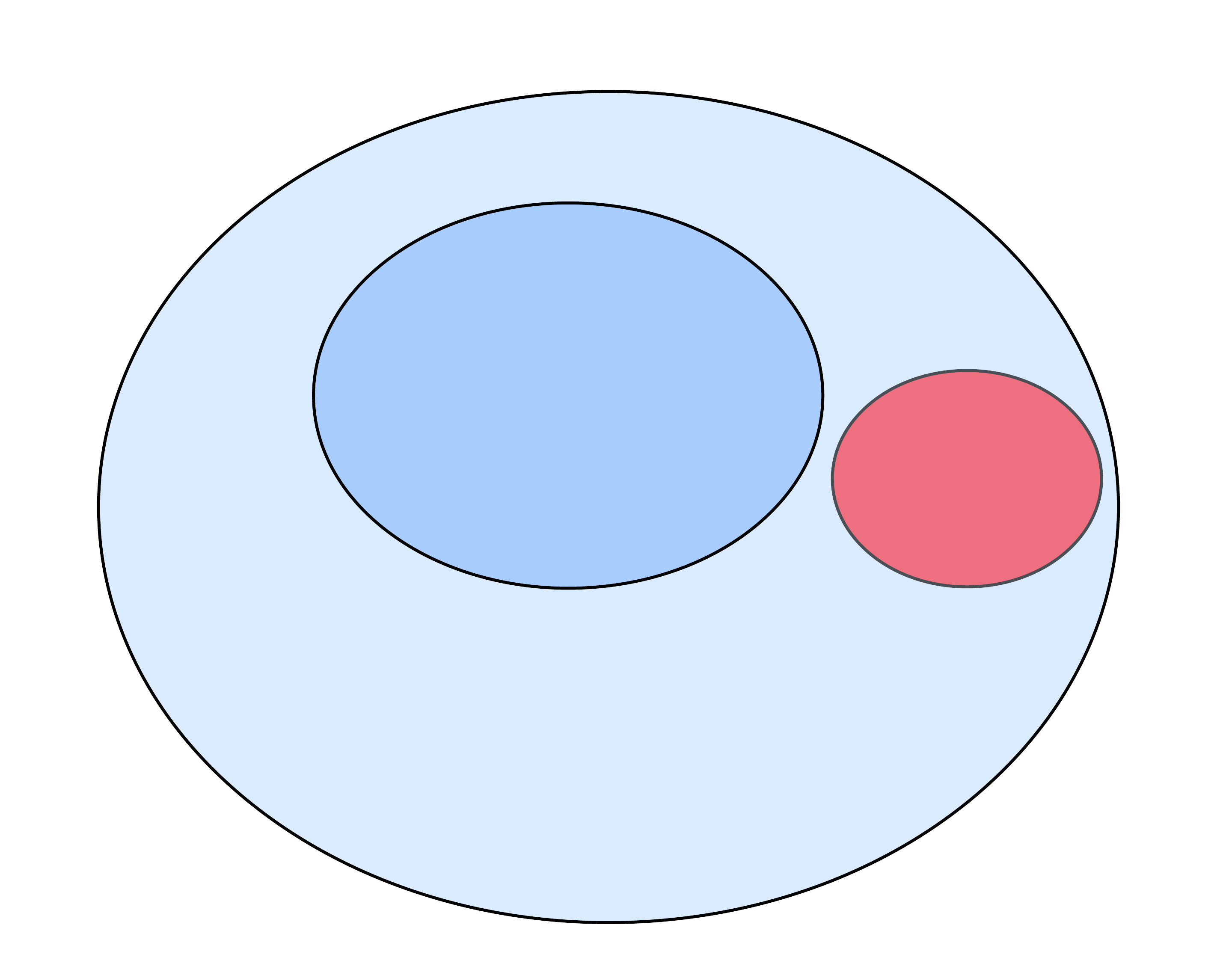}};
    \node[inner sep=0pt] (subw1) at (-0.3,-1.1) {$\text{SubW}\left(\frac{1}{\beta_1}\right)$};
    \node[inner sep=0pt] (subw2) at (0,0.4) {$\text{SubW}\left(\frac{1}{\beta_2}\right)$};
    \node[inner sep=0pt] (wt) at (1.88,0) {$\text{GWT}\left(\beta\right)$};
    \end{tikzpicture}
\end{minipage}
\begin{minipage}{0.49\textwidth}
\centering
    \begin{tikzpicture}
    \node[inner sep=0pt] (potatos) at (0,0)
    {\includegraphics[width=5cm]{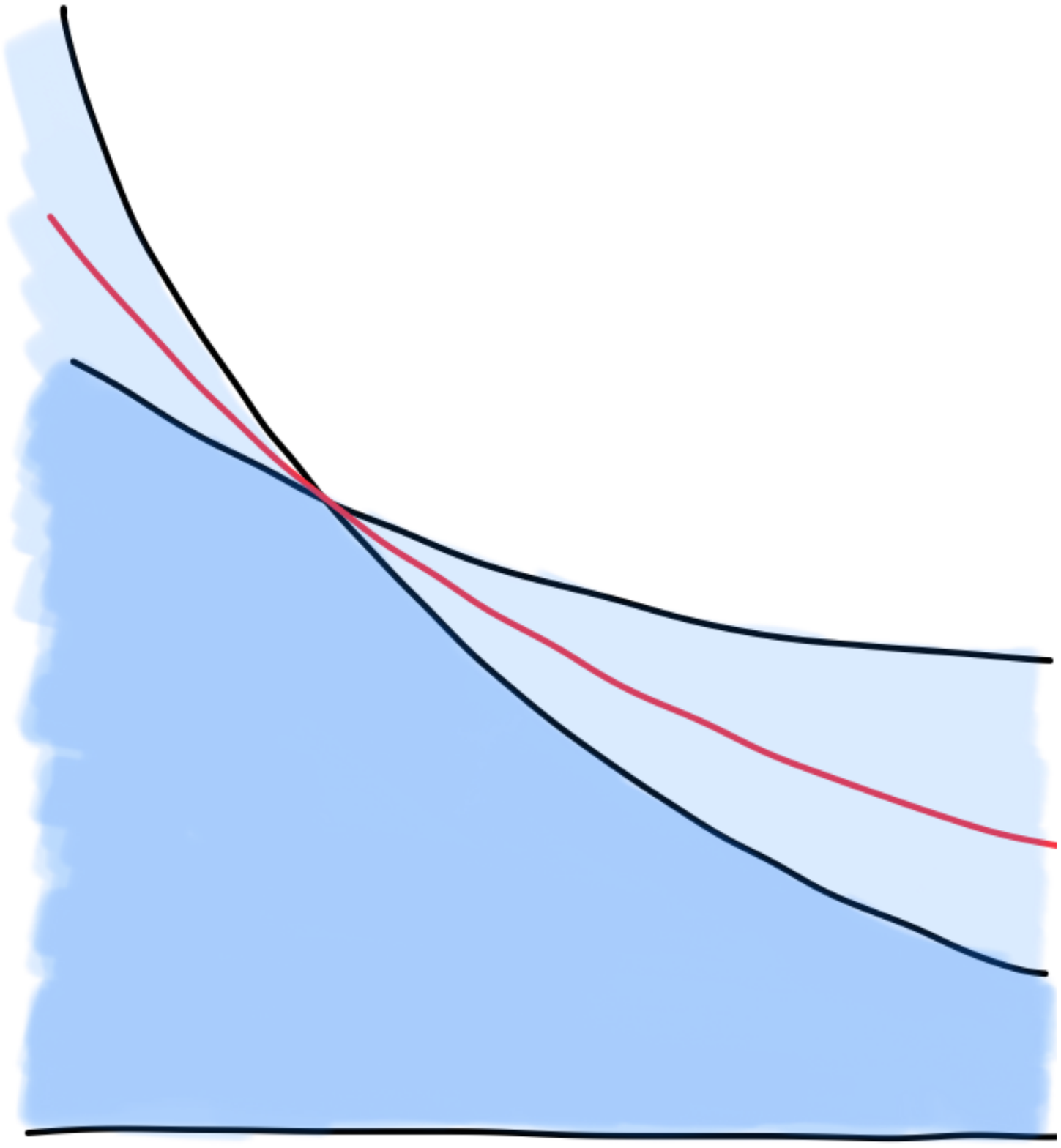}};
    \node[inner sep=0pt] (subw1) at (3.7,-0.5) {$\text{SubW}\left(\frac{1}{\beta_1}\right)$};
    \node[inner sep=0pt] (wt) at (3.5,-1.2) {\textcolor{purple}{$\text{GWT}\left(\beta\right)$}};
    \node[inner sep=0pt] (subw2) at (3.7,-1.9) {$\text{SubW}\left(\frac{1}{\beta_2}\right)$};
    \end{tikzpicture}
\end{minipage}}
\vspace{-1.4cm}
    \caption{Relation between sub-Weibull and generalized Weibull-tail characteristics.}
    \label{fig:potatos}
\end{figure}

It was recently shown in \citet{vladimirova2018bayesian} that hidden units of Bayesian neural networks with iid Gaussian priors are sub-Weibull with tail parameter proportional to the hidden layer number, that is $\theta = \frac{\ell}2$. It means that the unit distributions of hidden layer $\ell$ can be upper-bounded by some Weibull distributions $a \edr^{-x^{2/\ell}}$ for all $\ell$.  
For larger tail parameter $\theta$, Weibull distribution $a \edr^{-x^{1/\theta}}$ is heavier-tailed but being sub-Weibull does not guarantee the heaviness of the tails. However, this upper bound is optimal in the sense that it is achieved for neural networks with one hidden unit per layer.

From Theorem~\ref{theorem:hidden_units_are_gwt}, for neural networks with independent Gaussian weights, hidden units of $\ell$-th layer are generalized Weibull-tail with tail parameter $\beta = 1/\theta = 2/\ell$ so they have upper and lower bounds of the form $\edr^{-x^{2/\ell}l(x)}$ up to a constant where $l$ is some slowly-varying function. Therefore, it proves that hidden units are heavier-tailed as going deeper for any finite numbers of hidden units per layer.

\subsection{Meijer G-functions description}

In~\citet{springer1970distribution} it was shown that the probability density function of the product of independent normal variables could be expressed through a Meijer G-function. It resulted in an accurate description of induced unit priors given Gaussian priors on weights and linear or ReLU activation function ~\citep{zavatone2021exact,noci2021precise}.  It is a full description of function-space priors but under strong assumptions, requiring Gaussian priors on weights and linear or ReLU activation functions, and with convoluted expressions. In contrast, we provide results for many distributions, including heavy-tailed ones, and our results can be extended to smooth activation functions, such as PReLU, ELU, Softplus.

\section{Future applications}
\label{section:future_applications}

\paragraph*{Cold posterior effect and priors.}
\label{subsection:cold_posterior}

It was recently empirically found that Gaussian priors led to the 
\textit{cold posterior effect} in which a tempered ``cold'' posterior, obtained by exponentiating the posterior to some power largely greater than one, performs better than an untempered one~\citep{wenzel2020good}. The performed Bayesian inference is considered sub-optimal due to the need for cold posteriors, and the model is deemed misspecified. From that angle, \citet{wenzel2020good} suggested that Gaussian priors might not be a good choice for Bayesian neural networks. 
In some works, data augmentation is argued to be the main reason for this effect~\citep{izmailov2021bayesian,nabarro2021data} as the increased amount of observed data naturally leads to higher posterior contraction~\citep{izmailov2021bayesian}. 
At the same time, even considering the data augmentation for some models, the cold posterior effect is still present. 
In addition, \citet{aitchison2020statistical} demonstrates that the problem might originate in the wrong likelihood of the models and that modifying only the likelihood based on data curation mitigates the cold posterior effect. 
\citet{nabarro2021data} hypothesize that using an appropriate prior incorporating knowledge of data augmentation might provide a solution. Moreover, heavy-tailed priors have been shown to mitigate the cold posterior effect~\citep{fortuin2021bayesian}.
According to Theorem~\ref{theorem:hidden_units_are_gwt}, heavier-tailed priors lead to even heavier-tailed induced priors in function-space. 
Thus, the heavy-tail property of distributions in function-space might be a highly beneficial feature. 
\citet{fortuin2021bayesian} also proposed correlated priors for convolutional neural networks since trained weights are empirically strongly correlated. Correlated priors improve overall performance but do not alleviate the cold posterior effect.  Our theory can be extended to  correlated weight priors. This direction is promising for further uncovering the effect of weight prior on function-space prior.

\paragraph*{Edge of Chaos.}
An active line of research studies the propagation of deterministic inputs in neural networks~\citep{poole2016exponential,schoenholz2016deep,hayou2019impact}. The main idea is to explore the covariance between pre-activations for two given different data points. \citet{poole2016exponential} and \citet{schoenholz2016deep} obtained recurrence relations under the assumption of Gaussian initialization and Gaussian pre-activations. They conclude that there is a critical line, so-called \textit{Edge of Chaos}, separating signal propagation into two regions. The first one is an ordered phase in which all inputs end up asymptotically correlated. The second is a chaotic phase in which all inputs end up asymptotically independent. To propagate the information deeper in a neural network, one should choose Gaussian prior variances corresponding to the separating line. \citet{hayou2019impact} show that the smoothness of the activation function also plays an important role. 
Since this line of works considers Gaussian priors not only on the weights but also on the pre-activations, it is closely related to a wide regime where the number of hidden units per layer tends to infinity. Given that hidden units are heavier-tailed with depth, we speculate that future research will focus on finding better approximations of the pre-activation functions in  recurrence relations  obtained for finite-width neural networks.

\section{Conclusion}

We extend the theory on induced distributions in Bayesian neural networks and establish an accurate and easily interpretable characterization of hidden units tails. The obtained results confirm the heavy-tailed nature of hidden units for different weight priors.

% \acks{Acknowledgements should go at the end, before appendices and references.}

%\bibliographystyle{plain}
% \bibliography{biblio}

% \appendix

\appendix

\section{Slowly and regularly-varying functions theory}

The set of regular-varying functions with index $\rho \in \mathbb{R}$ is denoted by $\mathcal{RV}_\rho$. Note that for $\rho = 0$, the set $\mathcal{RV}_0$ boils down to the set of slowly-varying functions. In particular, any function  $r \in \mathcal{RV}_\rho$ can be written $r(x) = x^\rho l(x)$, where $l$ is slowly-varying. 

\begin{definition}[Regularly-varying function]
Let $r$ be a  positive function. Then $r\in\mathcal{RV}_\rho$ if for all $t > 0$
% \begin{equation*}
    $\lim_{x\to\infty} \frac{r(tx)}{r(x)} = t^\rho$.
% \end{equation*}
\end{definition}

\begin{prop}{\citep[Proposition 1.3.6]{bingham1989regular}}
\label{appendix:proposition:sv_properties}
Let $l, l_1,\ldots,l_k$ be slowly-varying functions. Then:
\begin{enumerate}
    \item $(\log l(x))/\log x \to 0$ as $x \to \infty$.
    \item $l^\alpha$ varies slowly for every $\alpha \in \mathbb{R}$.
    \item $l_1 l_2$, $l_1 + l_2$, and (if $l_2(x) \to \infty$ as $x \to \infty$) $l_1 \circ l_2$ vary slowly.
    \item If  $f(x_1, \dots, x_k)$ is a rational function with positive coefficients, $f(l_1, \dots, l_k)$ varies slowly.
    \item For any $\alpha > 0$, 
    $x^\alpha l(x) \to \infty, \ x^{-\alpha} l(x) \to 0 \ (x\to \infty)$.
\end{enumerate}
\end{prop}

\begin{lm}
\label{lemma:sv_of_power}
If $l_1(x)$ is slowly-varying, then $l_2(x) = l_1(x^a)$ is slowly-varying for $a > 0$.
\end{lm}
% \begin{proof}
% For any $t > 0$ we have 
% \begin{equation*}
%     \lim_{x\to \infty} \frac{l_2(tx)}{l_2(x)} = \lim_{x\to \infty} \frac{l_1(t^a x^a)}{l_1(x^a)} = \lim_{y \to \infty} \frac{l_1(t^a y)}{l_1(y)} = 1.
% \end{equation*}
% \end{proof}

\begin{lm}
\label{lemma:sv_min_max}
If $l_1, l_2$ vary slowly, so does $\max\{l_1, l_2\}$ and $\min\{l_1, l_2\}$. 
\end{lm}

\begin{lm}
\label{lemma:sum_of_exponents_of_regular-varying}
Let $r_1$ and $r_2$ be regularly-varying functions with parameters $\beta_1 > 0$ and $\beta_2 > 0$. 
Then, the function $r$ such that $\edr^{-r} = \edr^{-r_1} + \edr^{-r_2}$, is regularly-varying with parameter $\beta = \min \{\beta_1, \beta_2\}$. 

\end{lm}
\begin{proof}
Let us express function $r$ from the statement: 
% Equation~\eqref{lemma:eq:sum_of_exponents_of_regular-varying}:
% \begin{equation*}
    $r(x) = - \log \varphi(x) = - \log \left( \edr^{-r_1(x)} + \edr^{-r_2(x)} \right)$.
% \end{equation*}
If $\beta_1 \not= \beta_2$, without loss of generality, let us assume that $\beta_1 < \beta_2$, then \begin{equation*}
    \varphi(x) = \edr^{-r_1(x)} \left( 1 + \edr^{-r_2(x) + r_1(x)} \right) = \edr^{-r_1(x)} \left( 1 + \edr^{r_2(x)\left( -1 + \frac{r_1(x)}{r_2(x)}\right)} \right).
\end{equation*}
Notice that $\frac{r_1}{r_2} \in \mathcal{RV}_{\beta_1 - \beta_2}$. 
For $\beta_1 < \beta_2$, $\frac{r_1(x)}{r_2(x)} \to 0$ when $x \to \infty$. The expression in~the exponent $r_2(x)\left( -1 + \frac{r_1(x)}{r_2(x)}\right) \simeq\footnote{We say functions $f \simeq g$ if and only if $f / g \to 1$.} -r_2(x) \to -\infty$, so the exponent $\edr^{r_2(x)\left( -1 + \frac{r_1(x)}{r_2(x)}\right)} \to 0$. Then $\varphi(x) \simeq \edr^{-r_1(x)}$ and $r(x) = - \log \varphi(x) \simeq r_1(x)$. It means that for the case $\beta_1 < \beta_2$, $r \in \mathcal{RV}_{\beta_1}$. 

Let us consider the case with equal parameters $\beta_1 = \beta_2 = \beta$, then $r_1(x) = x^{\beta} l_1(x)$, $r_2(x) = x^{\beta} l_2(x)$ with slowly-varying $l_1$ and $l_2$. 
With $l = \min\{l_1, l_2\}$, we can write
\begin{equation*}
    \varphi(x) = \edr^{-x^\beta l(x)} \left( 1 + \edr^{ - x^\beta|l_2(x) - l_1(x)|}\right).
\end{equation*}
Consider the logarithm of the latter expression
\begin{equation*}
    - \log \varphi(x) = x^\beta \left( l(x) + x^{-\beta} \log \left( 1 + \edr^{ - x^\beta|l_2(x) - l_1(x)|}\right) \right).
\end{equation*}
Since the function $l$ is slowly-varying by Lemma~\ref{lemma:sv_min_max} and $0 \le \edr^{ - x^\beta|l_2(x) - l_1(x)|} \le 1$, then $r(x) = -\log \varphi(x) \simeq x^\beta l(x) \in \mathcal{RV}_{\beta}$. 
\end{proof}

% \begin{lm}[Power and multiplication by a constant]
% \label{lemma:gen_weibull-tail_power}
% Let non-negative random variable~$X$ be Weibull-tail (generalized Weibull-tail) with tail parameter $\beta$, then $aX^b$ is Weibull-tail (generalized Weibull-tail) with tail parameter $\frac{\beta}b$ for $a, b > 0$.
% \end{lm}
% \begin{proof}
% For $a, b > 0$, the tail of $Y = aX^b$ is $\mathbb{P} \Bigl( a X^b \ge y \Bigr) = \mathbb{P} \Bigl( X \ge \left(\frac{y}{a}\right)^{1/b} \Bigr)$.

% If $X$ is Weibull-tail with tail parameter $\beta$, then $\mathbb{P} \left( X \ge x \right) = \edr^{-x^\beta l(x)}$, 
% where $l$ is a slowly-varying function. It implies 
% \begin{equation*}
%      \mathbb{P} \Bigl( a X^b \ge y \Bigr) = \edr^{-y^{\beta/b} \tilde l(y)},
% \end{equation*}
% where $\tilde l(y) = \frac{l((y/a)^{1/b})}{a^{\beta/b}}$ is a  slowly-varying function. 

% If $X$ is generalized Weibull-tail with tail parameter $\beta$, then  $\edr^{-x^\beta l_1(x)} \le \mathbb{P} \left( X \ge x \right) \le \edr^{-x^\beta l_2(x)}$, 
% where $l_1$ and $l_2$ are slowly-varying functions. It implies 
% \begin{equation*}
%     \edr^{-y^{\beta/b} \tilde l_1(y)} \le \mathbb{P} \Bigl( a X^b \ge y \Bigr) \le \edr^{-y^{\beta/b} \tilde l_2(y)},
% \end{equation*}
% where $\tilde l_i(y) = \frac{l_i((y/a)^{1/b})}{a^{\beta/b}}, i = 1, 2$ are slowly-varying functions. 
% \end{proof}
\vspace{-0.7cm}
\section{Weibull-tail properties on $\mathbb{R}$}
\label{appendix:weibull-tail_properties_proofs}

Let us firstly introduce a notion of generalized Weibull-tail random variable which has an additional property of stability: 
\begin{definition}[Generalized Weibull-tail on $\mathbb{R}_+$]
\label{def:generalized_wt}
A random variable $X$ is called \textit{generalized Weibull-tail} with tail parameter $\beta > 0$ if its survival function $\overline{F}_X$ is bounded by Weibull-tail functions of tail parameter $\beta$ with possibly different slowly-varying functions $l_1$ and $l_2$:
\begin{equation}
\label{eq:generalized_weibull-tail_def}
   \edr^{-x^{\beta} l_1(x)} \le \overline{F}_X(x) \le \edr^{-x^{\beta} l_2(x)}, \,\,\, \text{for }x>0.
\end{equation}
We note $X \sim \text{GWT}_{\mathbb{R}_+}(\beta)$.
\end{definition}

Now we define a random variable whose both right and left tails are Weibull-tail on $\mathbb{R}_+$. 
\begin{definition}[Weibull-tail on $\mathbb{R}$]
\label{def:double_weibull-tail}
A random variable $X$ on $\mathbb{R}$ is Weibull-tail on $\mathbb{R}$ with tail parameter $\beta > 0$ if both its right and left tails are Weibull-tail with tail parameter $\beta$:
\begin{equation*}
    \overline{F}_X(x) = 1 - F_X(x) = \edr^{-x^\beta l_1(x)}, \quad x > 0
\end{equation*}
\begin{equation*}
    F_X(x) = \edr^{-|x|^\beta l_2(|x|)}, \quad x < 0,
\end{equation*}
where $l_1$ and $l_2$ are slowly-varying functions. 
We note $X \sim \text{WT}_{\mathbb{R}}(\beta)$.
\end{definition}

\begin{lm}
\label{lemma:doubleWT_and_absolute_value}
\begin{enumerate}
\item[(i)] 
If random variable $X$ on $\mathbb{R}$ is $\text{WT}_{\mathbb{R}}(\beta)$, then $|X|$ is $\text{WT}_{\mathbb{R}_+}(\beta)$. 
\item[(ii)] If $X$ is asymmetric but both tails are Weibull-tail and $|X|$ is $\text{WT}_{\mathbb{R}_+}(\beta)$, then one of the tails (right or left) is $\text{WT}_{\mathbb{R}_+}(\beta)$ and the other one is $\text{WT}_{\mathbb{R}_+}(\beta')$ where $\beta' \ge \beta$. 

\item[(iii)] For symmetric distributions, $X$ is $\text{WT}_{\mathbb{R}}(\beta)$ if and only if $|X|$ is $\text{WT}_{\mathbb{R}_+}(\beta)$.
\end{enumerate}
\end{lm}
\begin{proof}
\begin{enumerate}
\item[(i)] 
For $x > 0$, the cumulative distribution function of $|X|$ is the following
\begin{align*}
    F_{|X|}(x) = F_X(x) - F_X(-x).
\end{align*}
Then, $\overline{F}_{|X|}(x)$ can be expressed as a sum of the right and left tails:
\begin{align*}
    \overline{F}_{|X|}(x) =  \overline{F}_{X}(x) + F_X(-x).
\end{align*}
If $X$ is $\text{WT}_{\mathbb{R}}(\beta)$, then $|X|$ is $\text{WT}_{\mathbb{R}_+}(\beta)$ as a consequence of Definition~\ref{def:double_weibull-tail} and Lemma~\ref{lemma:sum_of_exponents_of_regular-varying}.

\item[(ii)] Let $|X|$ be $\text{WT}_{\mathbb{R}_+}(\beta)$ and $X$ has Weibull left and right tails with different tail parameters. Without loss of generality, assume that ${F}_{X}(-x)$ is Weibull-tail with tail parameter $\beta' < \beta$. According to Lemma~\ref{lemma:sum_of_exponents_of_regular-varying}, the sum survival function $\overline{F}_{|X|}(x)$ will be Weibull-tail with tail parameter $\beta' = \min\{\beta', \beta\}$. We obtained a contradiction and  ${F}_{X}(-x)$ must have tail parameter greater or equal $\beta$. If both tail parameters are greater than $\beta$, then the tails sum have the tail parameter equal to the minimum tail parameter among them which is greater than $\beta$. It means that at least one tail must have tail parameter $\beta$.

\item[(iii)] For symmetric distributions $\overline{F}_{X}(x) = {F}_{X}(-x)$ for any $x$, then $\frac12 \overline{F}_{|X|}(x) = \overline{F}_{X}(x) = {F}_{X}(-x)$ and we have the equality. 
\end{enumerate}
\end{proof}
\vspace{-0.7cm}
\begin{lm}
\label{lemma:GWT_and_absolute_value}
\begin{enumerate}
    \item[(i)] If a random variable $X$ is $\text{GWT}_{\mathbb{R}}(\beta)$, then $|X|$ is $\text{GWT}_{\mathbb{R}_+}(\beta)$.   \item[(ii)] For symmetric distributions, $X$ is $\text{GWT}_{\mathbb{R}}(\beta)$  if and only if $|X|$ is $\text{GWT}_{\mathbb{R}_+}(\beta)$.
\end{enumerate}
\end{lm}
\begin{proof}
Similarly as in Lemma~\ref{lemma:doubleWT_and_absolute_value}, we obtain $\overline{F}_{|X|}(x) =  \overline{F}_{X}(x) + F_X(-x)$ for $x \ge 0$.

\begin{enumerate}
    \item[(i)] If $X$ is $\text{GWT}_\mathbb{R}(\beta)$, then right and left tails are upper and lower-bounded by some Weibull-tail functions. Then, the sum of the right and left tails is upper and lower-bounded by these Weibull-tail functions and $|X| \sim \text{GWT}_{\mathbb{R}_+}(\beta)$. 

\item[(ii)] For symmetric distributions we have $\overline{F}_{X}(x) = {F}_{X}(-x)$ for all $x$, then $\frac12 \overline{F}_{|X|}(x) = \overline{F}_{X}(x) = {F}_{X}(-x)$ and we have the equality. 
\end{enumerate}
\end{proof}

\vspace{-0.7cm}
%The following lemma introduces multiplication by a constant for generalized Weibull-tail random variables on $\mathbb{R}$. 
\begin{lm}[Power and multiplication by a constant]
\label{lemma:gen_weibull-tail_power}
If $X \sim \text{GWT}_{\mathbb{R}}(\beta)$ and the distribution of $X$ is symmetric, then $a|X|^b \sim \text{GWT}_{\mathbb{R}_+} \left( \frac{\beta}b\right)$ for $a, b > 0$. 
\end{lm}
\begin{proof}
According to Lemma~\ref{lemma:doubleWT_and_absolute_value}, $|X| \sim \text{GWT}_{\mathbb{R}_+}(\beta)$. 
% It means also that for $a, b > 0$ the tail parameters of $aX^b$ and $a|X|^b$ must be equal as they are equal for $X$ and $|X|$. 
For $a, b > 0$, the tail of $Y = a|X|^b$ is 
\begin{equation*}
     \mathbb{P} \Bigl( a |X|^b \ge y \Bigr) = \mathbb{P} \left( |X|^b \ge \frac{y}{a} \right) = \mathbb{P} \Bigl( |X| \ge \left(\frac{y}{a}\right)^{1/b} \Bigr).
\end{equation*}
Since $|X|$ is generalized Weibull-tail on $\mathbb{R}_+$ with tail parameter $\beta$, $\edr^{-x^\beta l_1(x)} \le \mathbb{P} \left( |X| \ge x \right) \le \edr^{-x^\beta l_2(x)}$, 
where $l_1$ and $l_2$ are slowly-varying functions, it implies 
\begin{equation*}
     \edr^{-y^{\beta/b} \tilde l_1(y)} \le \mathbb{P} \Bigl( a |X|^b \ge y \Bigr) \le \edr^{-y^{\beta/b} \tilde l_2(y)},
\end{equation*}
where $\tilde l_i(y) = \frac{l_i((y / a)^{1/b})}{ a^{\beta/b}}$, $i = 1, 2$ are slowly-varying functions by Lemma~\ref{lemma:sv_of_power}.
It leads to the statement of the lemma.
\end{proof}

\begin{manualtheorem}{\ref{theorem:sum_of_N_gwt_rvs_on_r}}[Sum of GWT$_\mathbb{R}$ variables]
\theoremSumNGwt
\end{manualtheorem}
\begin{proof}
Let us start with $N=2$. For any random variables $X$ and $Y$, the following upper bound holds:
\begin{equation*}
    \mathbb{P}(X + Y \ge z) \le \mathbb{P}(X \ge \nicefrac{z}{2}) + \mathbb{P}(Y \ge \nicefrac{z}{2}) \le 2 \max\{\mathbb{P}(X \ge \nicefrac{z}{2}), \mathbb{P}(Y \ge \nicefrac{z}{2}) \}.
\end{equation*}

The PD condition leads to a lower bound for the sum: 
\begin{equation*}
    \mathbb{P}(X + Y \ge z) \ge \mathbb{P}(X \ge 0, Y \ge z) = \mathbb{P}(X \ge 0 | Y \ge z) \mathbb{P}(Y \ge z) \ge C \, \mathbb{P}(Y \ge z),
\end{equation*}
where constant $C > 0$. 
Thus, the sum survival function $\overline{F}_Z (z) = \mathbb{P}(X + Y \ge z)$ has the following bounds for the right tail: 
\begin{equation*}
    C \,\overline{F}_Y(z) \le \overline{F}_Z(z) \le 2 \max\{\overline{F}_X(\nicefrac{z}{2}), \overline{F}_Y(\nicefrac{z}{2})\},
\end{equation*}
where $\overline{F}_X$ and $\overline{F}_Y$ are the survival functions of $X$ and $Y$.

Let $X$ and $Y$ be generalized Weibull-tail on $\mathbb{R}$ of parameters $\beta_x$ and $\beta_y$, then for $z > 0$ large enough
\begin{equation*}
    C\, \edr^{-z^{\beta_y}l_1^r(z)} \le \overline{F}_Z(z) \le 2 \edr^{-z^{\min\{\beta_x, \beta_y\}}l_2^r(z)},
\end{equation*}
where $l_1^r = l_1^y$ is the slowly-varying function appearing in the right tail lower bound of generalized Weibull-tail $Y$ and $l_2^r(z) = \min\{ \frac{1}{2^{\beta_x}} l_2^x(\nicefrac{z}{2}), \frac{1}{2^{\beta_y}}l_2^y(\nicefrac{z}{2})\}$ is the minimum among slowly-varying functions, where $l_2^x$ and $l_2^y$ are slowly-varying functions in the right tails upper bounds of $X$ and $Y$. According to Lemma~\ref{lemma:sv_min_max}, $l_2^r(z)$ is also slowly-varying.  Similarly, we can get bounds for the left tail. Therefore, $X + Y$ is generalized Weibull-tail on $\mathbb{R}$ with tail parameter $\beta = \min\{\beta_x, \beta_y\}$. 

Similarly as above when $N=2$, bounds for the right tail of the sum $Z = X_1 + \dots + X_N$ with  survival function $\overline{F}_Z = \mathbb{P}(Z \ge z)$ are: 
\begin{equation*}
    C \overline{F}_{N}(z) \le \overline{F}_Z(z) \le N \max\{\overline{F}_{1}(\nicefrac{z}{N}), \dots,  \overline{F}_{N}(\nicefrac{z}{N})\},
\end{equation*}
where $\overline{F}_{i}$ is the survival functions of $X_i$ and constant $C > 0$. The rest of the proof is identical to the one of the case $N=2$. 
\end{proof}
% \vspace{-0.3cm}
The case when distributions have only right tails (or only left tails), can be considered as a particular case of the last theorem: a sum of non-negative generalized Weibull-tail random variables is non-negative generalized Weibull-tail with tail parameter equal to the minimum among those of the terms. 

\begin{manualtheorem}{\ref{theorem:product_of_double_weilbull-tail_rvs}}
[Product of independent $\text{GWT}_{\mathbb{R}}$ variables]
% \label{theorem:product_of_double_weilbull-tail_rvs}
\theoremProductGwt
\end{manualtheorem}
\begin{proof}
Consider two independent symmetric generalized Weibull-tail random variables with tail parameters $\beta_x = \frac 1 {\theta_x}$ and $\beta_y = \frac 1 {\theta_y}$, $X\sim \text{GWT}_{\mathbb{R}} \left( \frac 1 {\theta_x} \right)$ and $Y \sim \text{GWT}_{\mathbb{R}} \left( \frac 1 {\theta_y} \right)$. From Lemma~\ref{lemma:doubleWT_and_absolute_value} and since random variables $X$ and $Y$ are symmetric, it is equivalent to $|X| \sim \text{GWT}_{\mathbb{R}_+}\left( \frac 1 {\theta_x} \right)$ and $|Y| \sim \text{GWT}_{\mathbb{R}_+} \left( \frac 1 {\theta_y} \right)$. 

The product of independent symmetric distributions is symmetric since $Z = XY =^d (-X)Y =^d -Z$. 
From Lemma~\ref{lemma:doubleWT_and_absolute_value}, $Z \sim \text{GWT}_{\mathbb{R}}(\beta)$ if and only if $|Z| \sim \text{GWT}_{\mathbb{R}_+}(\beta)$. 

Our goal is to show that for some slowly-varying functions $l_1$ and $l_2$, there exist upper and lower bounds for the survival function of $|Z|$ and $z$ large enough as follows:
\begin{align}
\label{eq:product_tail_equality}
    \edr^{- z^{\frac{1}{\theta_x + \theta_y}} l_1(z)} \le \overline{F}_{|Z|} (z) = \mathbb{P}(|XY| \ge z) \le \edr^{- z^{\frac{1}{\theta_x + \theta_y}} l_2(z)}. 
\end{align}

\begin{enumerate}
\item[(i)] \textit{Upper bound.} 
First, notice that from the concavity of the logarithm,  we have $\ln(pu + (1 - p)v) \ge p \ln u + (1 - p) \ln v$ for any $u, v > 0$ and $p \in (0, 1)$. Then $pu + (1 - p)v \ge u^p v^{1 - p}$. The change of variables $x = u^p$, $y = v^{1 - p}$ implies 
$p x^{\nicefrac 1 p} + (1 - p) y^{\nicefrac 1 {(1 - p)}} \ge xy$.
From the latter equation, an upper bound of the product tail is 
\begin{equation}
\label{eq:product_upper_bound}
    \mathbb{P}(|XY| \ge z) \le \mathbb{P} \left( p |X|^{\nicefrac1p} + (1 - p) |Y|^{\nicefrac1{(1 - p)}} \ge z \right).
\end{equation}
Lemma~\ref{lemma:gen_weibull-tail_power} implies that  $p|X|^{\nicefrac1p} \sim \text{GWT}_{\mathbb{R}_+} \left( \frac{p}{\theta_x} \right)$ and  $(1 - p)|Y|^{\nicefrac1{1 - p}} \sim \text{GWT}_{\mathbb{R}_+} \left( \frac{1 - p}{\theta_y} \right)$. 
Taking $p = \frac{\theta_x}{\theta_x + \theta_y}$ and $1- p = \frac{\theta_y}{\theta_x + \theta_y}$, yields a sum of two independent non-negative generalized Weibull-tail random variables with tail parameter $\frac{1}{\theta_x + \theta_y}$ on the right-hand side of Equation~\eqref{eq:product_upper_bound}.
By Theorem~\ref{theorem:sum_of_N_gwt_rvs_on_r}, this sum is generalized Weibull-tail with the same tail parameter~$\frac{1}{\theta_x + \theta_y}$. 
It means that there exists a slowly-varying function $l_2$ such that the tail of product absolute value $|XY|$ is upper-bounded by
\begin{equation}
\label{eq:product_tail_upper_bound}
     \mathbb{P}(|XY| \ge z) \le \edr^{- z^{\frac{1}{\theta_x + \theta_y}} l_2(z)}.
\end{equation}

\item[(ii)] \textit{Lower bound.} 
By independence of $|X|$ and $|Y|$ we have
\begin{equation*}
    \mathbb{P}(|XY| \ge z) \ge \mathbb{P} \left( |X| \ge z^{\frac{\theta_x}{\theta_x + \theta_y}} \right)  \mathbb{P} \left( |Y| \ge z^{\frac{\theta_y}{\theta_x + \theta_y}} \right).
\end{equation*}
Since $|X|$ and $|Y|$ are generalized Weibull-tail on $\mathbb{R}_+$, we can define function $l_1(z) = l_1^x(z^{\theta_x/{(\theta_x + \theta_y)}}) + l_1^y(z^{\theta_y/{(\theta_x + \theta_y)}})$ with $l_1^x$ and $l_1^y$ being slowly-varying functions in the lower bounds of generalized Weibull-tail $|X|$ and $|Y|$. Then, $l_1$ is slowly-varying by Lemma~\ref{lemma:sv_of_power} and we have 
\begin{equation}
\label{eq:product_tail_lower_bound}
    \mathbb{P}(|XY| \ge z) \ge  \edr^{-z^{\frac{1}{\theta_x + \theta_y}} l_1(z)}. 
\end{equation}

\end{enumerate}

Combining together Equations~\eqref{eq:product_tail_upper_bound} and~\eqref{eq:product_tail_lower_bound} and  Definition~\ref{def:generalized_wt} with Lemma~\ref{lemma:doubleWT_and_absolute_value} implies the statement of the theorem. 
\end{proof}
\vspace{-0.5cm}
\section{Bayesian neural network properties}
\label{appendix:additional_lemmas}

Proofs of Section~\ref{section:bnn}. 
\begin{manuallemma}{\ref{lemma:units_dependence_condition}}
\lemmaHiddenUnitsDependence
\end{manuallemma}
\begin{proof}
The joint probability for the right tail $\mathbb{P} \left(\bigcap_{i=1}^{N} W_i X_i \ge z_i \right)$ can be expressed as 
\begin{equation}
\label{lemma:units_dependence_condition:eq:integral}
    \int_{-\infty}^{\infty} \dots  \int_{-\infty}^{\infty} \mathbb{P} \left(\bigcap_{i=1}^{N} W_i x_i \ge z_i  \Bigl. \Bigr| \bigcap_{j=1}^{N} X_j = x_j \right) f(x_1, \dots, x_N) \, \ddr x_1 \dots \ddr x_N.
\end{equation}
Independence between $W_i$ and $X_j$ yields 
\begin{equation*}
    \mathbb{P} \left(\bigcap_{i=1}^{N} W_i x_i \ge z_i  \Bigl. \Bigr|  \bigcap_{j=1}^{N} X_j = x_j \right) = \mathbb{P} \left(\bigcap_{i=1}^{N} W_i x_i \ge z_i \right) = \prod_{i=1}^{N} \mathbb{P} \left(W_i x_i \ge z_i \right),
\end{equation*}
where the last equality is due to the mutual independence of weights $W_1, \dots, W_N$.
Let $z_1 = \dots = z_{N-1} = 0$ and $z_N = z$. 
If $x_i = 0$, the probability $\mathbb{P} \left(W_i x_i \ge 0 \right) = 1$. If $x_i \not= 0$, then, due to the symmetry of $W_i$, the probability  $\mathbb{P} \left(W_i x_i \ge 0 \right) = \frac12$. 
Thus, the following lower bound holds:
% of the probability which is not conditioned on the values of $x_i$: 
\begin{equation*}
    \mathbb{P} \left(W_i x_i \ge 0 \right) \ge \frac12  \quad \text{for all } i\in\{1, \dots, N - 1\}.
\end{equation*}

Notice that 
\begin{align*}
% \label{lemma:units_dependence_condition:eq:integral}
   \int_{-\infty}^{\infty} \dots  \int_{-\infty}^{\infty}  \mathbb{P} \left( W_N x_N \ge z \right) f(x_1, \dots, x_N) \, \ddr x_1 \dots \ddr x_N = \mathbb{P} \left(W_N X_N \ge z \right). 
\end{align*}
Substituting the latter equations into Equation~\eqref{lemma:units_dependence_condition:eq:integral} leads to the lower bound:
\begin{equation*}
    \mathbb{P} \left(\bigcap_{i=1}^{N-1} W_i X_i \ge 0, W_N X_N \ge z \right) 
   \ge    \frac{1}{2^{N-1}} \mathbb{P} \left(W_N X_N \ge z \right).
\end{equation*}
By the conditional probability definition, we have 
\begin{equation*}
    \mathbb{P} \left(\bigcap_{i=1}^{N - 1} W_i X_i \ge 0 \Bigl. \Bigr| W_N X_N \ge z  \right) = \frac{ \mathbb{P} \left(\bigcap_{i=1}^{N} W_i X_i \ge 0 \right)}{\mathbb{P} \left(W_N X_N \ge z \right)} \ge  \frac{1}{2^{N-1}}. 
\end{equation*}

The proof for the left tail is identical. 
\end{proof}

\vspace{-0.7cm}
\begin{manualtheorem}{\ref{theorem:hidden_units_are_gwt}}
\theoremHiddenUnitsAreGWT
\end{manualtheorem}
\begin{proof}
The goal is to show that $h_j^{(\ell)} \sim \text{GWT}_{\mathbb{R}_+}\left( \beta^{(\ell)} \right)$ where $\frac{1}{\beta^{(\ell)}} = \frac{1}{\beta^{(1)}_w} + \dots + \frac{1}{\beta^{(\ell)}_w}$. We proceed by induction on the layer depth $\ell$. 

\begin{enumerate}
    \item[(i)] \textit{First hidden layer} 
    
For $\ell = 1$, by Lemma~\ref{lemma:units_dependence_condition}, products $w_{ij}^{(1)} h_i^{(0)}$ for $i=1, \dots, H_1$ satisfy the positive dependence condition of Definition~\ref{def:positive_dependence_condition}, thus the sum $g_j^{(1)} = \sum_i  w_{ij}^{(1)} h_i^{(0)}$ is generalized Weibull-tail on $\mathbb{R}$ with tail parameter $\beta^{(1)} = \beta^{(1)}_w$. Since ReLU function does not change the right tail and  post-activations lie completely on $\mathbb{R}_+$, we have $h_j^{(1)} \sim \text{GWT}_{\mathbb{R}_+}(\beta^{(1)})$.

 \item[(ii)] \textit{Step of induction} 

Let $h_j^{(\ell - 1)} \sim \text{GWT}_{\mathbb{R}_+}\left( \beta^{(\ell - 1)} \right)$ where $\frac{1}{\beta^{(\ell - 1)}} = \frac{1}{\beta^{(1)}_w} + \dots + \frac{1}{\beta^{(\ell - 1)}_w}$.
Since $h_j^{(\ell - 1)}$ is non-negative and $w_{ij}^{(\ell)}$ are symmetric, their product is symmetric by symmetry of the weights: $h_j^{(\ell - 1)} w_{ij}^{(\ell)}$ has the same distribution as $h_j^{(\ell - 1)} \left( -w_{ij}^{(\ell)} \right) = - h_j^{(\ell - 1)} w_{ij}^{(\ell)}$. By Theorem~\ref{theorem:product_of_double_weilbull-tail_rvs} and independence of random variables $h_j^{(\ell - 1)}$ and $w_{ij}^{(\ell)}$, we obtain that $h_j^{(\ell - 1)} w_{ij}^{(\ell)} \sim \text{GWT}_{\mathbb{R}}( \beta^{(\ell)})$ where $\frac{1}{\beta^{(\ell)}} =  \frac{1}{\beta^{(\ell - 1)}} + \frac{1}{\beta^{(\ell)}_w} = \frac{1}{\beta^{(1)}_w} + \dots + \frac{1}{\beta^{(\ell - 1)}_w} +  \frac{1}{\beta^{(\ell)}_w}$. 

According to Lemma~\ref{lemma:units_dependence_condition},  $ w_{ij}^{(\ell)}h_i^{(\ell - 1)}$, $i = 1, \dots, H_\ell$ satisfy the dependence condition of Definition~\ref{def:positive_dependence_condition}. Thus, the sum $g_j^{(\ell)} = \sum_{i=1}^{H_\ell}  w_{ij}^{(\ell)}h_i^{(\ell - 1)}$ is symmetric generalized Weibull-tail with the same tail parameter. The ReLU activation does not impact the right tail, therefore, $h_j^{(\ell)} \sim \text{GWT}_{\mathbb{R}_+}(\beta^{(\ell)})$. 

\end{enumerate}
\end{proof}

\end{document}